\documentclass[twoside,11pt]{article}

\usepackage{jmlr2e1}

\usepackage{amsmath}
\usepackage{amssymb}

\usepackage{epsfig}
\usepackage{undertilde}
\usepackage{bbm}
\usepackage{verbatim}
\usepackage{graphics}
\usepackage{graphicx}
\usepackage{curves}

\usepackage{algorithmic}
\usepackage{algorithm}
\usepackage{amsfonts}
\usepackage{url}
\usepackage[mathscr]{euscript}

\renewcommand{\l}{\left}
\renewcommand{\r}{\right}

\newcommand{\given}{\text{ } \middle | \text{ }}
\newcommand{\argmax}{\operatornamewithlimits{argmax}}
\newcommand{\argmin}{\operatornamewithlimits{argmin}}

\newcommand{\conv}{\operatorname{conv}}
\newcommand{\conc}{\operatorname{conc}}

\newcommand{\temph}[1]{{\textsl{#1}}}

\DeclareSymbolFont{AMSb}{U}{msb}{m}{n}
\DeclareMathSymbol{\N}{\mathbin}{AMSb}{"4E}
\DeclareMathSymbol{\ZZ}{\mathbin}{AMSb}{"5A}
\DeclareMathSymbol{\Y}{\mathbin}{AMSb}{"59}
\DeclareMathSymbol{\U}{\mathbin}{AMSb}{"55}
\DeclareMathSymbol{\RR}{\mathbin}{AMSb}{"52}
\DeclareMathSymbol{\Q}{\mathbin}{AMSb}{"51}
\DeclareMathSymbol{\Prob}{\mathbin}{AMSb}{"50}
\DeclareMathSymbol{\HHH}{\mathbin}{AMSb}{"48}
\DeclareMathSymbol{\I}{\mathbin}{AMSb}{"49}
\DeclareMathSymbol{\C}{\mathbin}{AMSb}{"43}
\DeclareMathSymbol{\E}{\mathbin}{AMSb}{"45}
\DeclareMathSymbol{\C}{\mathbin}{AMSb}{"43}
\DeclareMathSymbol{\D}{\mathbin}{AMSb}{"44}
\DeclareMathSymbol{\OO}{\mathbin}{AMSb}{"4F}
\DeclareMathSymbol{\X}{\mathbin}{AMSb}{"58}
\DeclareMathSymbol{\LL}{\mathbin}{AMSb}{"4C}

\newcommand{\Fl}{\mathcal{F}}

\newcommand{\A}{\mathcal{A}}

\newcommand{\Cl}{\mathcal{C}}

\newcommand{\Z}{\mathcal{Z}}

\newcommand{\tla}{\tilde{\lambda}}
\ShortHeadings{Response-Based Approachability}{Bernstein and Shimkin}
\firstpageno{1}

\newtheorem{assumption}{Assumption}

\begin{document}

\title{Response-Based Approachability and its Application to
Generalized No-Regret Algorithms}

\author{\name
Andrey Bernstein \email aberenstein@gmail.com \\
Department of Electrical Engineering \\
Technion -- Israel Institute of Technology\\
Haifa 32000, ISRAEL
\AND
Nahum Shimkin \email shimkin@ee.technion.ac.il \\
Department of Electrical Engineering \\
Technion -- Israel Institute of Technology\\
Haifa 32000, ISRAEL
}


\editor{[Submitted for review, October 2013]}

\maketitle

\begin{abstract}
Approachability theory, introduced by Blackwell (1956), provides fundamental
results on repeated games with vector-valued payoffs, and has been usefully applied
since in the theory of learning in games and to learning algorithms in the online
adversarial setup.
Given a repeated game with vector payoffs, a target set $S$ is
approachable by a certain player (the agent) if he can ensure that the average payoff vector converges to that set no matter what his adversary opponent does.
Blackwell provided two equivalent sets of conditions for a convex set
to be approachable. The first (primary) condition is a geometric separation condition,
while the second (dual) condition requires that the set be {\em non-excludable},
namely that for every mixed action of the opponent there exists a mixed action of the
agent (a {\em response}) such that the resulting payoff vector belongs to $S$.
Existing approachability algorithms rely on
the primal condition and essentially require to compute at each stage a projection
direction from a given point to $S$.
In this paper, we introduce an approachability algorithm that
relies on Blackwell's {\em dual} condition. Thus, rather than projection, the algorithm relies
on computation of the response to a certain action of the opponent at each
stage.
The utility of the proposed algorithm is demonstrated by applying it to certain
generalizations of the classical regret minimization problem, which include regret minimization
with side constraints and regret minimization for global cost functions. In these problems,  computation of the required projections is generally complex but a response
is readily obtainable.
\end{abstract}

\section{Introduction} \label{sec:intro}

Consider a repeated matrix game with \emph{vector-valued} rewards that is played by two players,
the {\em agent} and the {\em adversary} or {\em opponent}.
In a learning context the agent may represent the learning algorithm, while the adversary stands
for an arbitrary or unpredictable learning environment.
For each pair of simultaneous actions $a$ and $z$ (of the agent and the opponent, respectively)
in the one-stage game, a reward
vector $r(a,z)\in \RR^{\ell}$, $\ell \geq 1$, is obtained. In Blackwell's approachability problem \citep{lit:black}, the agent's goal is to ensure that the long-term average reward vector {\em approaches} a given target set $S$, namely converges to $S$ almost surely in the point-to-set distance.
If that convergence can be ensured irrespectively of the opponent's strategy, the set
$S$ is said to be {\em approachable}, and a strategy of the agent that satisfies this property
is an approaching strategy (or algorithm) for $S$.

Blackwell's approachability results have been broadly used in theoretical work on learning in games,
including equilibrium analysis in repeated games with incomplete information \citep{lit:aumann},
calibrated forecasting \citep{lit:Foster99},
and convergence to correlated equilibria \citep{lit:sergiu1}.
The earliest application, however, concerned the
notion of \emph{regret minimization}, or \emph{no-regret strategies}, that was introduced in \citet{lit:hannan}. Even before Hannan's paper was published, it was shown in \citet{lit:black2}
that regret minimization can be formulated as a particular approachability problem,
which led to a distinct class of no-regret strategies.
More recently, approachability was used in \citet{lit:rust} to prove an extended no-regret result
for games with imperfect monitoring, while \citet{lit:sergiu} proposed an alternative formulation
of no-regret as an approachability problem (see Section \ref{sec:app}).
An extensive overview of approachability and no-regret in the context of learning is games
can be found in \citet{FudenbergLevine1998}, \citet{PaytonYoung2004}, and \citet{lit:pred}.
The latter monograph also makes the connection to the modern theory of on-line learning
and prediction algorithms.
In a somewhat different direction, approachability theory was applied in
\citet{lit:mannor04} to a problem of multi-criterion reinforcement
learning in an arbitrarily-varying environment.

Standard approachability algorithms require, at each stage of the game, the computation the direction
from the current average reward vector to a closest point in the target set $S$.
This is implied by Blackwell's \emph{primal} geometric separation condition, which is a sufficient condition for approachability of a target set.
For \emph{convex} sets, this step is equivalent to computing the \emph{projection direction} of the average reward onto $S$.
In this paper, we introduce an approachability algorithm that avoids this projection computation step.
Instead, the algorithm relies on availability of a {\em response map}, that assigns to each
mixed action $q$ of the opponent a mixed action $p$ of the agent so that $r(p,q)$, the expected reward
vector under these two mixed actions, is in $S$.
Existence of such a map is based on the Blackwell's \emph{dual} condition, which is
also a necessary and sufficient condition for approachability of a convex target set.

The idea of constructing an approachable set in terms of a general response map
was employed in \citet{lit:lehrer_solan07} (updated in \cite{lit:lehrer_solan13}), in the context of internal no-regret strategies.
An explicit approachability algorithm which is based on computing the response to
{\em calibrated  forecasts} of the opponent's actions has been proposed in
\citet{lit:perchet}, and further analyzed in \citet{lit:me_MOR}.
However, the algorithms in these papers are essentially based on the computation of calibrated  forecasts of the opponent's actions,
a task which is known to be computationally hard \citep{lit:cal_hp_hard}.
In contrast, the algorithm proposed in the present paper operates strictly in the payoff space, similarly to Blackwell's approachability algorithm.

The main motivation for the proposed algorithm comes from certain generalizations of the
basic no-regret problem, where the set to be approached is complex
so that computing the projection direction may be hard, while the response map is explicit
by construction.
These generalizations include the constrained regret minimization problem \citep{lit:shie_constr}, regret minimization with global cost functions \citep{lit:global}, regret minimization in variable duration repeated games \citep{lit:nahum_var}, and regret minimization in stochastic game models \citep{lit:bayes_env}. In these cases, the computation of a response reduces to computing a \emph{best-response} in the underlying regret minimization problem, and hence can be carried out efficiently. The application of our algorithm to some of these problems is discussed in Section \ref{sec:applic_rc} of this paper.


The paper proceeds as follows.
In Section \ref{sec:app} we review the approachability problem and existing approachability algorithms,
and illustrate the formulation of standard no-regret problems as approachability problems.
Section \ref{sec:algo} presents our basic algorithm and  establishes its approachability properties.
In Section \ref{sec:variants}, we provide an interpretation of certain aspects of the proposed algorithm,
and propose some variants and extensions to the basic algorithm.
Section \ref{sec:applic_rc} applies the proposed algorithms to generalized no-regret problems, including constrained regret minimization and online learning with global cost functions. We conclude the paper in Section \ref{sec:conc}.

\section{Review of Approachability and Related No-Regret Algorithms} \label{sec:app}

In this Section, we present the approachability problem and review Blackwell's approachability conditions. We further discuss existing approachability algorithms, and illustrate the application of the approachability framework to classical regret minimization problems.

\subsection{Approachability Theory}

Consider a repeated two-person matrix game, played between an agent and an arbitrary opponent.
The agent chooses its actions from a finite set $\A$,
while the opponent chooses its actions from a finite set $\Z$. At
each time instance $n = 1, 2, ...$, the agent selects its action
$a_n \in \A$, observes the action $z_n \in \Z$ chosen by the
opponent, and obtains a \emph{vector-valued} reward $R_n = r(a_n, z_n) \in
\RR^{\ell}$, where $\ell \geq 1$, and $r:\A \times \Z \rightarrow \RR^{\ell}$ is a given
reward function. The average reward vector obtained by the agent up to time $n$ is then
$
\bar{R}_n = n^{-1}\sum_{k = 1}^{n} R_k.
$
A \emph{mixed} action of the agent is a probability vector $p
\in \Delta(\A)$, where $p(a)$ specifies the probability of choosing
action $a \in \A$, and $\Delta(\A)$ denotes the set of probability vectors
over $\A$ . Similarly, $q \in \Delta(\Z)$ denotes a mixed action of the opponent. Let $\bar{q}_n \in
\Delta(\Z)$ denote the empirical distribution of the opponent's
actions at time $n$, namely
$$
\bar{q}_n(z) \triangleq \frac{1}{n} \sum_{k=1}^n \I\l\{z_n = z\r\}, \quad z \in \Z,
$$
where $\I$ is indicator function. Further define the Euclidean span of the reward vector as
\begin{equation} \label{asm:r_bound}
\rho \triangleq \max_{a,z,a',z'} \l\|r(a, z) - r(a', z') \r\|,
\end{equation}
where $\l\| \cdot \r\|$ is the Euclidean norm. The inner product between two vectors $v \in \RR^{\ell}$ and $w \in \RR^{\ell}$ is denoted by $v \cdot w$.

In what follows, we find it convenient to use the notation
$$
r(p, q) \triangleq \sum_{a \in \A, z \in \Z} p(a) q(z) r(a, z)
$$
for the expected reward under mixed actions $p \in \Delta(\A)$
and $q \in \Delta(\Z)$; the distinction between $r(a,z)$ and $r(p,q)$ should be clear from their arguments. Occasionally, we will use
$
r(p, z) = \sum_{a \in \A} p(a) r(a, z)
$
for the expected reward under mixed action $p \in \Delta(\A)$
and pure action $z \in \Z$.
The notation $r(a, q)$ is to be interpreted similarly.

Let
$$
h_{n} \triangleq \l\{a_1, z_1, ..., a_{n}, z_{n}\r\}
\in \l(\A \times \Z\r)^{n}
$$
denote the history of the game up to (and including) time $n$. A \emph{strategy} $\pi=(\pi_n)$
of the agent is a collection of decision rules $\pi_n:\l(\A \times \Z\r)^{n-1} \rightarrow \Delta(\A)$, $n \geq 1$, where each mapping $\pi_n$
specifies a mixed action $p_n = \pi_n(h_{n-1})$  for the agent at time $n$.
The agent's pure action $a_n$ is sampled from $p_n$.
Similarly, the opponent's strategy is denoted by $\sigma = (sigma_n)$, with $\sigma_n:\l(\A \times \Z\r)^{n-1} \rightarrow \Delta(\Z)$.
Let $\Prob^{\pi, \sigma}$ denote the probability measure on $\l(\A \times \Z\r)^{\infty}$
induced by the strategy pair $(\pi,\sigma)$.

Let $S$ be a given target set. Below is the classical definition of an approachable set from \citet{lit:black}.

\begin{definition}[Approachable Set] \label{def:appr_set}
A closed set $S \subseteq \RR^{\ell}$ is
\temph{approachable by the agent's strategy $\pi$}
if the average reward $\bar{R}_n = n^{-1}\sum_{k = 1}^{n} R_k$ converges to $S$ in the Euclidian point-to-set distance $d(\cdot, S)$, almost surely for every strategy $\sigma$ of the opponent, at a uniform rate over all strategies $\sigma$ of the opponent.
That is,
for every $\epsilon > 0$ there is an integer $N$ such that, for every strategy $\sigma$ of the opponent,
\[
\Prob^{\pi, \sigma} \l\{d\l(\bar{R}_n, S\r) \geq \epsilon \text{ for some } n \geq N \r\} < \epsilon.
\]
The set $S$ is \temph{approachable} if there exists such a strategy for the agent.
\end{definition}

In the sequel, we find it convenient to state most of our results in terms of the \emph{expected} average reward, where expectation is applied only to the agent's mixed actions:
\[
\bar{r}_n = \frac{1}{n}\sum_{k = 1}^n r_k \triangleq \frac{1}{n}\sum_{k = 1}^n r(p_k, z_k).
\]
With this modified reward, the stated convergence results will be shown to hold \emph{pathwise}, for any possible sequence of the opponent's actions.
See, e.g., Theorem \ref{theo:algo}, where we show that
$d(\bar{r}_n,S)\leq \frac{\rho}{\sqrt{n}}$ for all $n$.
The corresponding almost sure convergence for the actual average reward $\bar{R}_n$ can be easily deduced using martingale convergence theory.  Indeed, note that
\[
d\l(\bar{R}_n , S\r) \leq \l\|\bar{R}_n  - \bar{r}_n \r\| + d\l(\bar{r}_n, S\r).
\]
 But the first term is the norm of the mean of the \emph{vector}
 martingale difference sequence $D_k = r(a_k, z_k) - r(p_k, z_k)$. This can be easily shown to converge to zero at a uniform rate of $O\l(1/\sqrt{n}\r)$, 
 using standard results (e.g., from \citet{lit:shir}); see for instance \cite{lit:appr_stoch}, Proposition 4.1. In particular, it can be shown that there exists a finite constant $K$ so that for each $\delta > 0$
\[
\l\|\bar{R}_n  - \bar{r}_n \r\| \leq \frac{K \log (1/\delta)}{\sqrt{n}}
\]
with probability at least $1 - \delta$.

Next, we present a formulation of Blackwell's results \citep{lit:black} which provides us with
conditions for approachability of general and convex sets.
To this end, for any $x \notin S$, let $c(x) \in S$ denote a closest point in $S$ to $x$. Also, for any $p \in
\Delta(\A)$ let $T(p) \triangleq \l\{r(p, q): \text{ } q \in \Delta(\Z) \r\}$, which coincides with  the convex hull of the vectors $\l\{r(p, z)\r\}_{z \in \Z}$.
\begin{definition}\label{def:B-D-sets}
\begin{itemize}
\item[]
\item[(i)] \textbf{B-sets:}
A closed set $S \subseteq \RR^{\ell}$ will be called a \temph{B-set} (where \temph{B} stands for \temph{Blackwell}) if for every $x \notin S$ there exists a mixed action $p^* = p^*(x) \in \Delta(\A)$ such that the hyperplane through $y = c(x)$ perpendicular to the line segment $xy$, separates $x$ from $T(p^*)$.
\item[(ii)] \textbf{D-sets:}
A closed set $S \subseteq \RR^{\ell}$ will be called a \temph{D-set} (where \temph{D} stands for \temph{Dual}) if for every $q \in \Delta(\Z)$ there exists a mixed action $p \in \Delta(\A)$ so that $r(p, q) \in S$. We shall refer to
such $p$ as an \emph{$S$-response} (or just \emph{response}) of the agent to $q$.
\end{itemize}
\end{definition}

\begin{theorem}\label{theo:appr_primal}
\begin{enumerate}
\item[]
\item[(i)] \textbf{Primal Condition and Algorithm.}
 A B-set is approachable, by using at time $n$ the mixed action $p^*(\bar{r}_{n-1})$ whenever $\bar{r}_{n-1} \notin S$. If $\bar{r}_{n-1} \in S$, an arbitrary action can be used.
\item[(ii)] \textbf{Dual Condition.} A closed set $S$ is approachable only if it is a D-set.
\item[(iii)] \textbf{Convex Sets.} Let $S$ be a closed \temph{convex} set. Then, the
following statements are equivalent: (a) $S$ is approachable, (b) $S$ is a B-set, (c) $S$ is a D-set.
\end{enumerate}
\end{theorem}

We note that the approachability algorithm in Theorem \ref{theo:appr_primal}$(i)$ is valid also if $\bar{r}_n$ in the primal condition is replaced with $\bar{R}_n$. In addition, Theorem \ref{theo:appr_primal} has the following Corollary.
\begin{corollary} \label{col:d_set}
The \temph{convex hull} of a D-set is approachable (and is also a B-set).
\end{corollary}
\proof
The convex hull of a D-set is a convex D-set. The claim then follows by Theorem \ref{theo:appr_primal}.
\endproof

 Since Blackwell's original construction, some other approachability algorithms that
 are based on similar geometric ideas have been proposed in the literature. \citet{lit:sergiu} proposed a class of
 approachability algorithms that use a general steering direction with separation properties. As shown there, this is essentially equivalent to the computation of the projection to the target set in some norm. When Euclidean norm is used, the resulting algorithm is equivalent to Blackwell's original scheme. Recently, \citet{ABH10} proposed an elegant scheme which generates the required steering
 directions through a no-regret algorithm (in the online convex programming framework).
We provide in the Appendix a somewhat simplified version of that algorithm which
is meant to clarify the geometric basis of the algorithm, which involves the
\emph{support function} of the target set.


We mention in passing some additional theoretical results and extensions.
\citet{lit:vieille92} studied the weaker notions of weak approachability and
excludability, and showed that these notions are complimentary even for non-convex sets.
\citet{lit:spinat} formulated a necessary and sufficient condition for approachability of \emph{general} (not necessarily convex) sets.
In \citet{lit:appr_stoch} and \citet{lit:milman}, approachability was extended to
stochastic (Markov) game models.
An extension of approachability theory to infinite dimensional reward spaces
was carried out in \citet{lit:lehrer}, while
\citet{lit:lehrer_solan09} considered approachability strategies with bounded memory.

Recently, \citet{lit:robust_appr} proposed a robust approachability algorithm for repeated games with partial monitoring and applied it to the corresponding regret minimization problem.

 In all these papers, at each time step, either the computation of the projection to the target set, or that of a steering direction with separation properties is required.

\subsection{Approachability and No-Regret Algorithms} \label{sec:no_regret}
We next present the problem of regret minimization in repeated matrix games, and show how these problems
can be formulated in terms of approachability with an appropriately defined reward vector and target set.
We start with Blackwell's original formulation, and
proceed to the alternative one by \citet{lit:sergiu}. In the final subsection, we consider briefly the more elaborate problem of {\em internal} regret minimization.
We will mainly emphasize the role of the dual condition and the simple computation of the {\em response} for these problems, and refer to the
respective references for details of the (primal) resulting algorithms.

Consider, as before, the agent that faces an arbitrarily varying environment (the opponent).
The repeated game model is the same as above, except that the vector reward function $r$ is replaced by a scalar reward (or utility) function $u:\A \times \Z \rightarrow \RR$. Let $\bar{U}_n \triangleq n^{-1} \sum_{k = 1}^n U_k$ denote the average reward by time $n$, and let
\[
U_n^*(z_1, ..., z_n) \triangleq \frac{1}{n} \max_{a \in \A} \sum_{k
= 1}^n u(a, z_k)
\]
denote the \emph{best reward-in-hindsight} of the agent after
observing $z_1, ..., z_n$. That is, $U_n^*$ is the maximal average
reward the agent could obtain at time $n$ if he knew the opponent's actions
beforehand and used a single fixed action. It is not hard to see
that the best reward-in-hindsight can be defined as a \emph{convex}
function $u^*$ of the empirical distribution $\bar{q}_n$ of the opponent's
actions:
\begin{equation}
\label{eqn:rih_basic} U_n^*(z_1, ..., z_n)
=\max_{a \in \A} u(a,\bar{q}_n) \triangleq u^*(\bar{q}_n).
\end{equation}
This motivates the definition of the average \emph{regret} as $(u^*(\bar{q}_n) - \bar{U}_n)$,
and the following definition of a no-regret algorithm:
\begin{definition}[No-Regret Algorithm] \label{def:no_regret}
We say that a strategy of the agent is a \temph{no-regret algorithm} (also termed a \temph{Hannan
Consistent} strategy) if
\[
\limsup_{n \rightarrow \infty} \l(u^*(\bar{q}_n) - \bar{U}_n \r) \leq 0,
\]
almost surely, for any strategy of the opponent.
\end{definition}

\subsubsection{Blackwell's Formulation} \label{sec:no_regret_Black}

Following Hannan's seminal paper, \citet{lit:black2} used approachability theory in order to elegantly show the existence of regret minimizing algorithms.
Define the vector-valued rewards $R_n \triangleq \l(U_n, \textbf{1}(z_n)\r) \in \RR \times \Delta(\Z)$, where $\textbf{1}(z)$ is the probability vector in $\Delta(\Z)$ supported on $z$.
The corresponding average reward is then $\bar{R}_n \triangleq n^{-1} \sum_{k = 1}^n R_k = \l(\bar{U}_n, \bar{q}_n\r)$.
Finally, define the target set
\[
S = \l\{(u, q) \in \RR \times \Delta(\Z): \text{ } u
\geq u^*(q)\r\}.
\]
It is easily verified that this set is a D-set: by construction, for each $q$ there exists an
$S$-response $p \in \argmax_{p \in \Delta(\Z)} u(p, q)$ so that $r(p, q) = (u(p, q), q) \in S$, namely $u(p, q) \geq u^*(q)$. Also, $S$ is a convex set by the convexity of $u^*(q)$ in $q$. Hence, by Theorem \ref{theo:appr_primal}, $S$ is approachable, and by the continuity of $u^*(q)$, an algorithm that approaches $S$ also minimizes the regret in the sense of Definition \ref{def:no_regret}. Application of Blackwell's approachability strategy
to the set $S$ therefore results in a no-regret algorithm.
We note that the required projection of the average reward vector onto $S$ is somewhat implicit in this formulation.

\subsubsection{Regret Matching} \label{sec:reg_match}
An alternative formulation, proposed in \citet{lit:sergiu}, leads to a a  simple and explicit no-regret algorithm for this problem.
Let
\begin{equation} \label{eqn:regrets}
L_n(a') \triangleq \frac{1}{n} \sum_{k = 1}^n \l( u (a', z_k) -
u (a_k, z_k)\r)
\end{equation}
denote the regret accrued due to not using action $a'$ constantly up to
time $n$.
The no-regret requirement in Definition \ref{def:no_regret} is then equivalent to
\begin{equation} \label{eqn:no_regret_goal}
\limsup_{n \rightarrow \infty}  L_n(a) \leq 0\, \quad a\in \A
\end{equation}
almost surely, for any strategy of the opponent. In turn, this goal is equivalent to the approachability of the the non-positive orthant $S= \mathbb{R}_{-}^{\A}$ in the game with
vector payoff $r=(r_{a'})\in\mathbb{R}^{\A}$, defined as
$r_{a'}(a,z)=u(a',z)-u(a,z)$.

To verify the dual condition, observe that $r_{a'}(p,q) = u(a',q)-u(p,q)$. Choosing
$p\in \argmax_p u(p,q)$ clearly ensures $r(p,q)\in S$, hence is an $S$-response to $q$ (in the sense
of Definition \ref{def:B-D-sets}(ii)), and $S$ is a D-set. Note that the response
here can always be taken as a pure action.

It was shown in \citet{lit:sergiu} that application of Blackwell's approachability strategy
in this formulation leads to the so-called \emph{regret matching} algorithm, where the probability of action $a$ at time step $n$ is given by:
\begin{equation} \label{eqn:appr_no_regret}
p_n(a) = \frac{\l[L_{n-1}(a)\r]_{+} }{\sum_{a'
\in \A} \l[L_{n-1}(a')\r]_{+} }.
\end{equation}
Here, $[x_a]_+ \triangleq \max\{x_a, 0\}$.
In fact, using their generalization of Blackwell's approachability strategies, the authors of that paper obtained a whole class of no-regret algorithms
with different weighting of the components of $L_n$.

\subsubsection{Internal Regret} \label{sec:reg_match_int}
We close this section with another application of approachability to the stronger notion of {\em internal} regret. Given a pair of different actions $a, a' \in \A$, suppose the agent were to replace action $a$ with $a'$ every time $a$ was played in the past. His reward at time $k = 1, ..., n$ would become:
\[
W_k(a, a') \triangleq
\begin{cases}
u(a', z_k), & \text{ if } a_k = a,\\
u(a_k, z_k), & \text{ otherwise.}
\end{cases}
\]
The \emph{internal} average regret of the agent for not playing $a'$ instead of $a$ is then given by
\begin{equation} \label{eqn:regrets_int}
I_n(a, a') \triangleq \frac{1}{n} \sum_{k = 1}^n \l( W_k(a, a') -
U_k\r).
\end{equation}
A \emph{no-internal-regret strategy} must ensure that
\begin{equation} \label{eqn:int_min}
\limsup_{n \rightarrow \infty} \max_{a, a' \in \A} I_n(a, a') \leq 0.
\end{equation}

To show existence of such strategies, define the vector-valued reward function $r(a, z) \in \RR^{\A\times \A}$ by setting its $(a_1, a_2)$ coordinate to
\[
r_{a_1, a_2}(a, z) \triangleq
\begin{cases}
u(a_2, z) - u(a_1, z), & \text{ if } a = a_1, \\
0, & \text{ otherwise.}
\end{cases}
\]
Internal no-regret is then equivalent to approachability of the negative quadrant $S_0=\{r\leq 0\}$.
It is easy to verify that $S_0$ is a D-set, by pointing out the response map:
Given a mixed action $q$ of the opponent, choosing
$a^*\in \argmax_{a\in\A}u(a,q)$ clearly results in $r(a^*,q)\leq 0$. Therefore, By Theorem \ref{theo:appr_primal}$(iii)$, the set $S_0$ is approachable.

The formulation of internal-no-regret as the approachability problem above, along with explicit
approaching strategies, in due to \citet{lit:sergiu1}.  The importance of internal regret
in game theory stems from the fact that if each player in a repeated $N$-player game uses
such a no-internal regret strategy, then the empirical distribution of the players' actions
convergence to the set of correlated equilibria. Some interesting relations between
internal and external (Hannan's) regret are discussed in \citet{lit:BlumMansour07}.

\section{Response-Based Approachability} \label{sec:algo}

In this section, we present our basic algorithm and establish its approachability properties.


Throughout the paper, we consider a target set $S$ that satisfies
the following assumption.
\begin{assumption} \label{asm:cont}
The set $S$ is a convex and approachable set. Hence, by Theorem \ref{theo:appr_primal}, $S$ is a
D-set: For all $q \in \Delta(\Z)$ there exists an $S$-response $p \in \Delta(\A)$
such that $r(p, q) \in S$.
\end{assumption}
Under this assumption, we may define a {\em response map} $p^S:\Delta(\Z)\to\Delta(\A)$ that assigns to each mixed action $q$ a response $p^S(q)$ so that $r(p^S(q),q)\in S$.

We note that in some cases of interest, including those discussed in Section \ref{sec:applic_rc},
the target $S$ may itself be defined through an appropriate response map.
Suppose that for each $q \in \Delta(\Z)$, we are given a response $p^*(q)\in \Delta(\A)$, devised so that $r(p^*(q),q)$ satisfies some desired
properties. Then the set $S= \conv\{ r(p^*(q),q),\,q\in \Delta(\Z)\}$ is, by construction, a
convex D-set, hence approachable.


We next present our main results and the basic form of the related approachability algorithm.
The general idea is the following. By resorting to the response map,
we create a specific sequence of {\em target points} $(r^*_k)$ with $r^*_k\in S$.
Letting
\[
\bar{r}^*_n = \frac{1}{n}\sum_{k = 1}^n r^*_k
\]
denote the $n$-step average target point, it follows that $\bar{r}^*_n \in S$ by
convexity of $S$. Finally, the agents actions are chosen so that
the difference $\bar{r}^*_{n} - \bar{r}_{n }$ converges to zero, implying that
$\bar{r}_{n }$ converges to $S$.

Let
 \[
\lambda_{n} \triangleq \bar{r}^*_{n} - \bar{r}_{n}
\]
denote the difference between the average target vector and
the average reward vector.

\begin{theorem} \label{theo:algo}
Let $\lambda_0=0$. Suppose that at each time step $n\geq 1$, the agent chooses its mixed action $p_n$
(from which $a_n$ is sampled) and two additional mixed actions $q^*_n \in \Delta(\Z)$ and $p^*_n \in \Delta(\A)$ as follows:
\begin{itemize}
\item[(i)]
$p_n$ and $q^*_n$ are equilibrium strategies in the zero-sum game with payoff matrix defined by
$r(a,z)$ projected in the direction $\lambda_{n-1}$, namely,
\begin{equation} \label{eqn:algo}
p_n \in \argmax_{p \in \Delta(\A)} \min_{q \in \Delta(\Z)} \lambda_{n-1} \cdot r(p, q),
\end{equation}
\begin{equation} \label{eqn:algo_q}
q^*_n \in \argmin_{q \in \Delta(\Z)} \max_{p \in \Delta(\A)} \lambda_{n-1} \cdot r(p, q),
\end{equation}
\item[(ii)] $p_n^*$ is chosen as an $S$-response to $q_n^*$,
so that  $r(p^*_n, q^*_n)\in S$; set $r^*_n=r(p^*_n, q^*_n)$.
\end{itemize}
Then
\begin{equation} \label{eqn:conv1}
d\l(\bar{r}_{n}, S\r) \leq \l\|\lambda_n\r\|  \leq \frac{\rho}{\sqrt{n}},  \quad n \geq 1,
\end{equation}
for any strategy of the opponent.
\end{theorem}
Observe that the required choice of $p_n^*$ as an $S$-response to $q_n^*$ is possible
due to our standing Assumption \ref{asm:cont}.
The conclusion of this theorem clearly implies that the set $S$ is approached by the specified strategy, and provides an explicit bound on the rate of convergence.
The approachability algorithm implied by Theorem \ref{theo:algo} is summarized
in Algorithm \ref{algo:main}.

The computational requirements Algorithm  \ref{algo:main} are as follows.
The algorithm has two major computations at each time step $n$:
\begin{enumerate}
\item The computation of the $(p_n, q^*_n)$ -- the equilibrium strategies in the zero-sum matrix game with the reward function $\lambda_{n-1} \cdot r(p, q)$. This boils down to the solution of the related primal and dual linear programs, and hence can be done efficiently.
    Note that, given the vector $\lambda_{n-1}$, this computation does not involve the
    target set $S$.
\item The computation of the target point $r^*_n = r(p^*_n, q^*_n)$, which is problem dependent. For example, in the constrained regret minimization problem this reduces to the computation of a \emph{best-response} action to $q^*_n$. This problem is further discussed in Section \ref{sec:applic_rc}.
\end{enumerate}

The proof of the last Theorem follows from the next result, which also provides less
specific conditions on the required choice of $(p_n,q_n^*)$.
\begin{proposition} \label{prop:1}
\begin{itemize}
\item[]
\item[(i)] Suppose that at each time step $n\geq 1$, the agent chooses the triple
$(p_n,q_n^*,p_n^*)$ so that
\begin{equation} \label{eqn:gen}
\lambda_{n - 1} \cdot \l(r(p_n, z) - r(p^*_n, q^*_n) \r) \geq 0, \quad \forall z \in \Z,
\end{equation}
and sets $r_n^*= r(p^*_n, q^*_n)$.
Then it holds that
\[
\l\|\lambda_n\r\| \leq \frac{\rho}{\sqrt{n}} \quad \forall n \geq 1.
\]

\item[(ii)] If, in addition, $p_n^*$ is chosen as an $S$-response to $q_n^*$,
so that  $r^*_n=r(p^*_n, q^*_n)\in S$,  then
\begin{equation} \label{eqn:conv2}
d\l(\bar{r}_{n}, S\r) \leq \l\|\lambda_n\r\|  \leq \frac{\rho}{\sqrt{n}},  \quad n \geq 1,
\end{equation}
\end{itemize}
\end{proposition}
The specific choice of $(p_n,q_n^*)$ in equations \eqref{eqn:algo}-\eqref{eqn:algo_q}
satisfies the requirement in \eqref{eqn:gen}, as argued below. Indeed, the latter
requirement is less restrictive, and can replace \eqref{eqn:algo}-\eqref{eqn:algo_q} in the
definition of the basic algorithm. However, the former choice is convenient as
it ensures that \eqref{eqn:gen} holds for {\em any} choice of $p_n^*$.

\begin{algorithm}[t]
\caption{Response-Based Approachability}
\label{algo:main}
\begin{trivlist}
\item \textbf{Initialization:} At time step $n=1$, use arbitrary mixed action $p_1$ and set an arbitrary target point $r^*_1 \in S$.

\item \textbf{At time step $n = 2, 3, ...$}:
\begin{enumerate}
\item Set an approachability direction
\[
\lambda_{n-1} = \bar{r}^*_{n-1} - \bar{r}_{n - 1},
\]
where
\[
\bar{r}_{n-1} = \frac{1}{n-1}\sum_{k = 1}^{n-1} r(p_k, z_k),
\qquad
\bar{r}^*_{n-1} = \frac{1}{n-1}\sum_{k = 1}^{n-1} r^*_k
\]
are, respectively, the average (smoothed) reward vector and the average target point.
\item Solve a zero-sum matrix game with the scalar reward function $\lambda_{n-1} \cdot r(p, q)$. In particular, find the optimal mixed action $p_n$ and $q^*_n$ that satisfy
    \[
    p_n \in \argmax_{p \in \Delta(\A)} \min_{q \in \Delta(\Z)} \lambda_{n-1} \cdot r(p, q),
    \]
    \[
    q^*_n \in \argmin_{q \in \Delta(\Z)} \max_{p \in \Delta(\A)} \lambda_{n-1} \cdot r(p, q).
    \]
\item Choose action $a_n$ according to $p_n$.
\item Pick $p^*_n$ so that $r(p^*_n, q^*_n) \in S$, and set the target point
\[
r^*_n = r(p^*_n, q^*_n).
\]
\end{enumerate}
\end{trivlist}
\end{algorithm}

We proceed to the proof of Proposition \ref{prop:1} and Theorem \ref{theo:algo}.
We first establish a useful recursive relation for $\l\|\lambda_n \r\|^2$.
\begin{lemma} \label{lem:dist_bound}
For any $n \geq 1$, we have that
\[
n^2 \l\|\lambda_n\r\|^2 \leq (n-1)^2 \l\|\lambda_{n-1}\r\|^2  +
2 (n - 1) \lambda_{n-1} \cdot \l( r^*_n - r_n\r) + \rho^2.
\]
where $\rho$ is the span of the reward function (\ref{asm:r_bound}).
\end{lemma}

\begin{proof}
We have that
\begin{eqnarray*}
\l\|\bar{r}^*_{n} -  \bar{r}_n\r\|^2 &=& \l\|\frac{n-1}{n}\l(\bar{r}^*_{n - 1} - \bar{r}_{n-1}\r) + \frac{1}{n}\l(r^*_n - r_n\r)\r\|^2 \\
&=& \l(\frac{n-1}{n}\r)^2 \l\|\bar{r}^*_{n - 1} -  \bar{r}_{n-1}\r\|^2 + \frac{1}{n^2} \l\|r^*_n - r_n \r\|^2 \\
&& + 2 \frac{n-1}{n^2} \l(\bar{r}^*_{n - 1} - \bar{r}_{n-1}\r) \cdot \l( r^*_n - r_n\r) \\
&\leq& \l(\frac{n-1}{n}\r)^2 \l\|\bar{r}^*_{n - 1} -  \bar{r}_{n-1}\r\|^2 + \frac{\rho^2}{n^2}  \\
&& + 2 \frac{n-1}{n^2} \l(\bar{r}^*_{n - 1} - \bar{r}_{n-1}\r) \cdot \l( r^*_n - r_n\r),
\end{eqnarray*}
where $\rho$ is the reward bound defined in (\ref{asm:r_bound}). The proof is concluded by multiplying both sides of the inequality by $n^2$.
\end{proof}

\begin{proof}[Proof of Proposition~\ref{prop:1}.]
Under condition (\ref{eqn:gen}), we have for all $n$ that
\[
\lambda_{n-1} \cdot \l( r^*_n - r_n\r) = \lambda_{n - 1} \cdot \l(r(p^*_n, q^*_n) - r(p_n, z_n) \r) \leq 0.
\]
Hence, by Lemma \ref{lem:dist_bound},
\[
n^2 \l\|\lambda_n\r\|^2 \leq (n-1)^2 \l\|\lambda_{n-1}\r\|^2  + \rho^2, \quad n \geq 1.
\]
Applying this inequality recursively, we obtain that
\[
n^2 \l\|\lambda_n\r\|^2 \leq n \rho^2, \quad n \geq 1
\]
or
\[
\l\|\lambda_n\r\|^2 \leq \rho^2/n, \quad n \geq 1,
\]
as claimed in part (i). Part (ii) now follows since $r_n^*\in S$ (for all $n$) implies that
$\bar{r}_n^*\in S$ (recall that $S$ is a convex set), hence
$$
d\l(\bar{r}_{n}, S\r) \leq \|\bar{r}_n - \bar{r}_n^*\| =  \l\|\lambda_n\r\| \,.
$$
\end{proof}

\begin{proof}[Proof of Theorem~\ref{theo:algo}.]
It only remains to show that the choice of $(p_n,q_n^*)$ in equations \eqref{eqn:algo}-\eqref{eqn:algo_q} implies the required inequality in \eqref{eqn:gen}.
Indeed, under (\ref{eqn:algo}) and (\ref{eqn:algo_q}) we have that
\begin{eqnarray*}
\lambda_{n-1} \cdot r(p_n, z_n) &\geq& \max_{p \in \Delta(\A)} \min_{q \in \Delta(\Z)} \lambda_{n-1} \cdot r(p, q) \\
&=& \min_{q \in \Delta(\Z)} \max_{p \in \Delta(\A)} \lambda_{n-1} \cdot r(p, q) \\
&\triangleq& \max_{p \in \Delta(\A)} \lambda_{n-1} \cdot r(p, q^*_n) \\
&\geq& \lambda_{n-1} \cdot r(p^*_n, q^*_n),
\end{eqnarray*}
where the equality follows by the minimax theorem for matrix games.
Therefore, condition (\ref{eqn:gen}) is satisfied for any $p^*_n$, and in particular for the one satisfying $r(p^*_n, q^*_n) \in S$. This concludes the proof of the Theorem.
\end{proof}

\section{Interpretation and Extensions} \label{sec:variants}
We open this section with an illuminating interpretation of
the proposed algorithm
in terms of a certain approachability problem in an auxiliary game,
and proceed to present several variants and extensions to the basic algorithm.
While each of these variants is presented separately, they may also be combined
when appropriate.

\subsection{An Auxiliary Game Interpretation}
 \label{sec:conn}

A central part of Algorithm \ref{algo:main} is the choice of the pair $(p_n,q_n^*)$ so that
$\bar{r}_n$ tracks $\bar{r}^*_n$, namely $\lambda_n = \bar{r}^*_n - \bar{r}_n \to 0$  (see
Equations \eqref{eqn:algo}-\eqref{eqn:algo_q} and Proposition \ref{prop:1}).
If fact, the choice of $(p_n,q_n^*)$ in \eqref{eqn:algo}-\eqref{eqn:algo_q}
can be interpreted as Blackwell's strategy for a specific approachability problem
in an auxiliary game, which we define next.

Suppose that at time $n$, the agent chooses a \emph{pair} of actions $(a, z^*) \in \A \times \Z$ and the opponent chooses a pair of actions $(a^*, z) \in  \A \times \Z$. The vector payoff function, now denoted by $v$, is given by
\[
v((a, z^*), (a^*, z)) = r(a^*, z^*) - r(a, z),
\]
so that
\[
V_n = r(a^*_n, z^*_n) - R_n.
\]
Consider the single-point target set $S_0 = \{0\} \subset \RR^{\ell}$. This set is clearly convex, and we next show that it is a D-set in the auxiliary game. We need to show that for any $\eta \in \Delta(\A \times \Z)$ there exists $\mu \in  \Delta(\A \times \Z)$ so that $v(\mu, \eta) \in S_0$, namely $v(\mu, \eta)=0$.
That that end, observe that
\[
v(\mu, \eta) = r(p^*, q^*) - r(p,q)
\]
where $p$ and $q^*$ are the marginal distributions of $\mu$ on  $\A$ and $\Z$,
respectively, while $p^*$ and $q$ are the respective marginal distributions of $\eta$.
Therefore we obtain $v(\mu, \eta)=0$ by choosing $\mu$ with the same marginals
as $\eta$, for example $\{\mu(a,z)=p(a)q^*(z)\}$ with $p=p^*$ and $q^*=q$.
Thus, by Theorem \ref{theo:appr_primal}, $S_0$ is approachable.

We may now apply Blackwell's approachability strategy to this auxiliary game.
Since $S_0$ is the origin, the direction from $S_0$ to the average reward $\bar{V}_{n-1}$ is just the average reward vector itself. Therefore, the primal (geometric separation) condition here is equivalent to
\[
\bar{V}_{n-1} \cdot v(\mu, \eta) \leq 0, \quad \forall\,\eta \in \Delta(\A \times \Z)
\]
or
\[
\bar{V}_{n-1} \cdot (r(p^*, q^*) - r(p,q)) \leq 0, \quad \forall\, p^* \in \Delta(\A), q\in  \Delta(\Z).
\]
Now, a pair $(p,q^*)$ that satisfies this inequality is any pair of equilibrium
strategies in the zero-sum game with reward $v$ projected in the direction of $\bar{V}_{n-1}$.
That is, for
\begin{equation}\label{eqn:aux1}
p \in \argmax_{p \in \Delta(\A)} \min_{q \in \Delta(\Z)} \bar{V}_{n-1} \cdot r(p, q),
\end{equation}
\begin{equation}\label{eqn:aux2}
q^* \in \argmin_{q \in \Delta(\Z)} \max_{p \in \Delta(\A)} \bar{V}_{n-1} \cdot r(p, q),
\end{equation}
it is easily verified that
\[
\bar{V}_{n-1} \cdot r(p^*, q^*) \geq \bar{V}_{n-1} \cdot r(p, q),
\quad \forall\, p^* \in \Delta(\A), q \in \Delta(\Z)
\]
as required.

The choice of $(p_n,q_n^*)$ in Equations \eqref{eqn:algo}-\eqref{eqn:algo_q}
follows \eqref{eqn:aux1}-\eqref{eqn:aux2}, with $\lambda_{n-1}$
replacing $\bar{V}_{n-1}$. We note that the two are not identical, as
$\bar{V}_{n}$ is the temporal average of $V_n=r(a_n^*,z_n^*)-r(a_n,z_n)$ while
$\lambda_{n}$ is the average the smoothed difference $r(p_n^*,q_n^*)-r(p_n,z_n)$;
however this does not change the approachability result above, and in fact
either can be used. More generally, any approachability algorithm
in the auxiliary game can be used to choose the pair $(p_n,q_n^*)$
in Algorithm \ref{algo:main}.

We note that in our original problem, the mixed action $q^*_n$ is not chosen
by an ``opponent" but rather specified as part of Algorithm
\ref{algo:main}. But since the approachability result above holds for an
arbitrary choice of $q^*_n$, it also holds for this particular one.

%

We proceed to present some additional variants of our algorithm.

\subsection{Idling when Inside $S$}
Recall that in the original approachability algorithm of Blackwell, an {\em arbitrary} action $a_n$ can be chosen by the agent whenever $\bar{r}_{n-1} \in S$. This may  reduce the computational burden of the algorithm, and adds another
degree of freedom that may be used to optimize other criteria.

Such arbitrary choice of $a_n$ (or $p_n$) when the average reward is in $S$
is also possible in our algorithm.
However, some care is required in the setting of the average target
point $\bar{r}^*_n$ over these time instances, as otherwise the two terms of the
difference $\lambda_n=\bar{r}^*_n-\bar{r}_n$ may drift apart.
As it turns out, what is required is simply to shift the average target point
$\bar{r}^*_n$ to $\bar{r}_n$ at these time instances, and use the modified
point in the computation of the steering direction $\lambda_n$.
In recursive form, we obtain the following modified recursion:
\begin{eqnarray}
&&\tla_0 = 0,\notag\\
\label{eqn:lam}&&\tla_n = \begin{cases}
\frac{n-1}{n}\tla_{n - 1} + \frac{1}{n}(r^*_n - r_n), & \text{ if } \bar{r}_n \notin S \\
0, & \text{ if } \bar{r}_n \in S,
\end{cases}
\quad n \geq 1.
\end{eqnarray}
It may be seen that the  steering direction $\tla_n$ is reset to $0$ whenever the average reward is in $S$. With this modified definition, we are able to maintain
the same convergence properties of the algorithm.

\begin{proposition} \label{theo:main_lam}
Let Assumption \ref{asm:cont} hold. Suppose that the agent uses Algorithm \ref{algo:main} with the following modifications:
\begin{enumerate}
\item The steering direction $\lambda_{n-1}$ is replaced by the modifed
direction $\tla_{n-1}$ defined recursively in (\ref{eqn:lam});
\item Whenever $\bar{r}_{n-1} \in S$, an arbitrary action $a_n$ is chosen.
\end{enumerate}
Then, it holds that
\[
d\l(\bar{r}_{n}, S\r) \leq \frac{\rho}{\sqrt{n}}\,, \quad n \geq 1,
\]
for any strategy of the opponent.
\end{proposition}

\proof
We establish the claim in two steps. We first show that $\|\tla_n\|$ bounds the Euclidean distance of $\bar{r}_n$ from $S$. We then show that $\|\tla_n\|$ satisfies an analogue of Lemma \ref{lem:dist_bound}, and therefore the analysis of the previous section holds.

To see that $d(\bar{r}_n, S) \leq \|\tla_n\|$ for all $n$, observe  that if $\bar{r}_n \in S$, then trivially $d(\bar{r}_n, S) = \|\tla_n\| = 0$. Assume next that $\bar{r}_n \notin S$. Let $n_0 < n$ be the last instant $n$ such that $\bar{r}_{n_0} \in S$. Using the abbreviate notation
$$ \bar{r}_{m:n} = \frac{1}{n-m+1}\sum_{k=m}^n r_k \,,$$
and similarly for $\bar{r}^*_{m:n}$, we obtain
\begin{eqnarray*}
\tla_n &=& \frac{1}{n} \l(\sum_{k = n_0 + 1}^n r^*_k -  \sum_{k = n_0 + 1}^n r_k\r)\\
&=& \frac{n - n_0}{n}\l(\bar{r}_{n_0+1:n}-  \bar{r}^*_{n_0+1:n} \r).
\end{eqnarray*}
On the other hand,
\begin{eqnarray*}
d(\bar{r}_n, S) &=& d (\frac{n_0}{n} \bar{r}_{n_0} +  \frac{n - n_0}{n} \bar{r}_{n_0+1:n}, S) \\
 &\leq& \frac{n_0}{n} d(\bar{r}_{n_0}, S) + \frac{n - n_0}{n} d \l( \bar{r}_{n_0+1:n}, S\r)\\
&\leq& 0 + \frac{n - n_0}{n} \l\|\bar{r}_{n_0+1:n} - \bar{r}^*_{n_0+1:n} \r\| \\
&=& \l\| \tla_n \r\|,
\end{eqnarray*}
where the first inequality follows by the convexity of the point-to-set Euclidean distance to a convex set, and the second inequality holds since
\[
\bar{r}^*_{n_0+1:n} = \frac{1}{n-n_0}\sum_{k = n_0 + 1}^n r^*_k \in S.
\]

For the second step, note that the recursive definition (\ref{eqn:lam}) of $\tla_n$ implies, similarly to the proof of Lemma \ref{lem:dist_bound}, that
\[
n^2 \l\|\tla_n\r\|^2 \leq (n - 1)^2 \l\|\tla_{n-1}\r\|^2 + 2 (n - 1) \tla_{n-1} \cdot \l( r^*_n - r_n\r) + \rho^2.
\]
Hence, when $\bar{r}_{n-1} \in S$, we have that $\tla_{n-1} = 0$, and arbitrary $a_n$ and $r^*_n$ can be chosen. Also, similarly to the analysis in Section \ref{sec:algo}, whenever $\bar{r}_{n-1} \notin S$, the solution $(p_n, q^*_n)$ of the zero-sum game in the direction $\tla_{n-1}$ ensures that
\[
\tla_{n-1} \cdot \l( r(p^*_n, q^*_n) - r(p_n, z_n)\r) \leq 0,
\]
and thus the convergence of $\l\|\tla_n\r\|$ to zero is implied.
\endproof

\subsection{Directionally Unbounded Target Sets} \label{sec:inf}
In some applications of interest, the target set $S$ may be unbounded in
certain directions. Indeed, this is the case in the approachability formulation of
the no-regret problem, where the goal is essentially to make the average reward as large as possible. In particular, in Blackwell's formulation
(Section \ref{sec:no_regret_Black}), the set $S=\{(u,q):u\geq u^*(q)\}$ is unbounded
in the direction of the first coordinate $u$. In Hart and Mas-Collel's formulation
(Section \ref{sec:reg_match}), the set $S=\{L\leq 0\}$ is unbounded in the
negative direction of all the coordinates of $L$.

In such cases, the requirement that $\lambda_n = \bar{r}^*_n-\bar{r}_n \to 0$,
which is a property of our basic algorithm, may be too strong, and may even
be counter-productive. For example, in Blackwell's no-regret formulation mentioned above, we would like to increase the first coordinate of $\bar{r}_n$ as much
as possible, hence allowing negative values of $\lambda_n$ makes sense (rather
than steering that coordinate to $0$ by reducing $\bar{r}_n$).
We propose next a modification of our algorithm that addresses this issue.

Given the (closed and convex) target set $S\subset \RR^{\ell}$,
let $D_S$ be the set of vectors $d\in \RR^{\ell}$ such that
$d+S\subset S$. It may be seen that $D_S$ is a closed and convex cone,
which trivially equals $\{0\}$ if (and only if) $S$ is bounded.
We refer to the unit vectors in $D_S$ as directions in which $S$ is unbounded.

Referring to the auxiliary game interpretation of our algorithm in Section \ref{sec:conn}, we may now relax the requirement that $\lambda_n$ approaches
$\{0\}$ to the requirement that $\lambda_n$ approaches $-D_S$. Indeed, if
we maintain $\bar{r}^*_n\in S$ as before, then $\lambda_n \in -D_S$ suffices
to verify that $\bar{r}_n = \bar{r}^*_n- \lambda_n \in S$.

We may now apply Blackwell's approachability strategy to the cone $D_S$
in place of the origin. The required modification to the algorithm is simple:
replace the steering direction $\lambda_n$ in  \eqref{eqn:algo}-\eqref{eqn:algo_q} or
(\ref{eqn:gen}) with the direction from the closest point in $-D_S$ to $\lambda_n$:
$$
\tla_{n} = \lambda_{n} - \mathrm{Proj}_{-D_S}(\lambda_{n})
$$
That projection is particularly simple in case $S$ is unbounded
along primary coordinates, so that the cone $D_S$ is a quadrant,
generated by a collection ${e_j, j\in J}$ of orthogonal unit vectors.
In that case, clearly,
$$
\mathrm{Proj}_{-D_S}(\lambda) =-  \sum_{j\in J} (e_j\cdot \lambda)^- \,.
$$
Thus, the negative components of $\lambda_{n}$ in directions $(e_j)$ are
nullified.

The modified algorithm admits analogous bounds to those of the basic algorithm,
with \eqref{eqn:conv1} or \eqref{eqn:conv2} replaced by
 \[
d\l(\bar{r}_{n}, S\r) \leq d(\lambda_n,-D_S)  \leq \frac{\rho}{\sqrt{n}},  \quad n \geq 1.
\]
The proof is similar and will thus be omitted.

\subsection{Using the Non-smoothed Rewards}
In the basic algorithm of Section \ref{sec:algo}, the definition of the steering direction $\lambda_{n}$ employs the smoothed rewards $r(p_k,z_k)$ rather than the
actual ones, namely $R_k=r(a_k,z_k)$. We consider here the case where the latter
are used. This is essential in case that the opponent's action $z_k$ is not
observed, so that $r(p_k,z_k)$ cannot be computed, but rather the reward
vector $R_k$ is observed directly. It also makes sense since the quantity
we are actually interested in is the average reward $\bar{R}_n$, and not
its smoothed version $\bar{r}_n$.

Thus, we replaced $\lambda_{n-1}$ with
\[
\tla_{n-1} = \bar{r}^*_{n-1} - \bar{R}_{n - 1}.
\]
The rest of the algorithm is the same as Algorithm \ref{algo:main}. We have the following result for this variant.

\begin{theorem} \label{theo:main_actual}
Suppose that Assumption \ref{asm:cont} holds. Then, if the agent uses Algorithm \ref{algo:main}, with $\lambda_{n-1}$ replaced by
\[
\tla_{n-1} = \bar{r}^*_{n-1} - \bar{R}_{n - 1}.
\]
it holds that
\[
\lim_{n \rightarrow \infty} \|\tla_n\| = 0,
\]
almost surely, for any strategy of the opponent, at a uniform rate of $O(1/\sqrt{n})$ over all strategies of the
opponent. More precisely, for every $\delta > 0$, we have that
\begin{equation} \label{eqn:conv3}
\Prob \l\{\max_{k \geq n}\|\tla_k\| \leq \sqrt{\frac{6 \rho^2}{\delta n}} \r\} \geq 1 - \delta.
\end{equation}
\end{theorem}
\proof{}
First observe that Lemma \ref{lem:dist_bound} still holds if $r_n = r(p_n, z_n)$ is replaced with $R_n = r(a_n, z_n)$ throughout. Namely,
\[
n^2 \|\tla_n\|^2 \leq (n-1)^2 \|\tla_{n-1}\|^2  + 2 (n - 1) \tla_{n-1} \cdot \l( r^*_n - r(a_n, z_n)\r) + \rho^2, \quad n \geq 1.
\]
Let $\l\{\Fl_n\r\}$ denote the filtration induced by the history. We have that
\begin{eqnarray}
\E \l[ n^2 \|\tla_n\|^2 \given \Fl_{n - 1} \r] &\leq& (n-1)^2 \|\tla_{n-1}\|^2  + 2 (n - 1) \tla_{n-1} \cdot \E \l[\l( r^*_n - r(a_n, z_n)\r)\given \Fl_{n - 1} \r] + \rho^2 \notag\\
&=& (n-1)^2 \|\tla_{n-1}\|^2  + 2 (n - 1) \tla_{n-1} \cdot \l( r^*_n - \E \l[ r(a_n, z_n) \given \Fl_{n - 1} \r]\r) + \rho^2\notag\\
\label{eqn:bound_stoch} &\leq& (n-1)^2 \|\tla_{n-1}\|^2  + \rho^2,
\end{eqnarray}
where the equality follows since $q^*_n$ and $p^*_n$ are determined by the history up to time $n-1$ and hence so does $r^*_n = r(p^*_n, q^*_n)$, and the last inequality holds since
\[
\tla_{n-1} \cdot \l( r^*_n - \E \l[ r(a_n, z_n) \given \Fl_{n - 1} \r]\r) \leq 0,
\]
similarly to the proof of Theorem \ref{theo:algo}. Now, we can proceed as in the original proof of Blackwell's theorem (\citet{lit:black}, Theorem 1) or use Proposition 4.1 in \citet{lit:appr_stoch} to deduce that a sequence $\|\tla_n\|$ satisfying (\ref{eqn:bound_stoch}) converges to zero almost surely at a uniform rate that depends only on $\rho$. In particular, using the proof of Proposition 4.1 in \citet{lit:appr_stoch} we know that for every $\epsilon > 0$ and $\delta > 0$, there exists $N = N(\epsilon, \delta, \rho)$ so that
\[
\Prob \l\{\exists n \geq N: \text{ } \|\tla_n\| \geq \epsilon \r\} \leq \delta,
\]
where $N$ can be chosen to be any constant greater than $\frac{6 \rho^2}{\delta \epsilon^2}$. This completes the proof of the Theorem.
\endproof

\section{Generalized No-Regret Algorithms} \label{sec:applic_rc}

The proposed approachability algorithms can be usefully applied to several generalized regret minimization problems, in which the computation of a projection onto the target
set is involved, but a response is readily obtainable. We start by providing
a generic description of the problem using a general set-valued goal function,
and then specialize the discussion to some specific goal functions that have
been considered in the recent literature.
We do not consider convergence rates in these section, but rather focus on asymptotic convergence results. Convergence rates can readily be derived by referring to our bounds in the previous sections; see, e.g., \eqref{eqn:conv1} or \eqref{eqn:conv3}.

Consider a repeated matrix game as before, where the vector-valued reward $r(a,z)$ is replaced with $v(a,z)\in \RR^K$.
Suppose that for each mixed action $q$ of the opponent, the agent has a
\emph{satisficing}\footnote{Borrowing from H.~Simon's terminology for
achieving satisfying (or good-enough) results in decision making.} \emph{payoff set} $V^*(q)\subset\RR^K$,
and at least one mixed action $p=p^*(q)$ that satisfies $v(p,q)\in V^*(q)$.
We refer to any such action as a \emph{response} of the agent to $q$.
Let $V^*:q\in\Delta(\Z) \mapsto V^*(q)$ denote the corresponding set-valued
{\em goal function}.
As before, let $V_n=v(a_n,z_n)$ and $\bar{V}_n=\frac{1}{n}\sum_{k=1}^n V_k$.
 A generalized no-regret strategy for
this model may be defined as strategy for the agent that ensures
$$\lim_{n\to\infty} d(\bar{V}_n, V^*(\bar{q}_n))=0 \quad \text{(a.\ s.)} $$
for any strategy of the opponent.
If such a strategy exists, we say that the goal function $V^*$ is {\em attainable}
by the agent.

The classical no-regret problem is obtained as a special case, with scalar reward
$v(a,z) = u(a, z)$ and satisficing payoff set
$V^*(q)=\{v\in \RR : v\geq v^*(q)\}$, where $v^*(q)\triangleq \max_p v(p,q)$.
As shown in Section \ref{sec:no_regret_Black},
this problem can be formulated as a particular case of approachability to the set $S=\{(v,q):v\in V^*(q)\}$, and existence of approaching strategies relies on convexity of the
function $v^*(q)$, which implies convexity of $S$.

A similar line of reasoning may be pursued for the generalized no-regret problem described above.
The no-regret property is clearly equivalent to approachability of the set
$S=\{(v,q):v\in V^*(q)\}$, in the game with reward vector $r(p,q)=(v(p,q),q)$.
As convexity of $S$, hence of $V^*$, plays an important role, we recall the following
definition for set-valued functions.
\begin{definition}[Convex hull]
A set valued function $V:q\in\Delta(\Z) \mapsto V(q)\subset \RR^K$ is \temph{convex} if
$\alpha V(q_1)+(1-\alpha)V(q_2)\subset V(\alpha q_1+(1-\alpha) q_2)$ for any
$q_1$, $q_2$ and $\alpha\in [0,1]$ (where the first plus sign stands for the Minkowski sum).
The \temph{convex hull} $V^c$ of $V$ is the minimal set-valued function
which is convex and contains $V$, in the sense that $V(q)\subset V^c(q)$ for each $q$.
Note that a minimal member (in the sense of set inclusion) exists, as the required property is invariant under intersections.
\end{definition}

The following claims follow easily from the definition of $S$.
%
\begin{proposition}\label{prop:GenRegret}
\begin{itemize}\item[]
\item[(i)]
The set $S=\{(v,q):v\in V^*(q)\}$ is a D-set. Hence, its convex hull
$\conv(S)$ is approachable.
\item[(ii)]
If the set-valued goal function $V^*$ is convex, then it is
attainable by the agent. In general, the convex hull $V^c$ of $V^*$ is
attainable by the agent.
\end{itemize}
\end{proposition}
\proof
To see that $S$ is a D-set, note that by its definition, for any $q$ there exists
$p$ such that $v(p,q)\in V^*(q)$, hence $(v(p,q),q)\in S$.
Therefore $\conv(S)$ is a convex D-set, which is approachable by Theorem \ref{theo:appr_primal}. Claim (ii) now follows by verifying that
$\conv(S)=\{(v,q):v\in V^c(q)\}$.
\endproof

It follows that any {\em convex} goal function $V^*$ is attainable. When
$V^*$ is not convex, which is often the case, one may need to resort to a
relaxed goal function, namely the convex hull $V^c$.
The computation of $V^c$ and its suitability as a (relaxed) goal function need to
be examined for each specific problem.

As a consequence of Proposition \ref{prop:GenRegret}, $V^c$ (or $V^*$
itself when convex) can be attained by any approachability algorithm applied to
the convex set $\conv(S)=\{(v,q):v\in V^c(q)\}$.
However, the required projection unto that set may be hard
to compute. This is especially true when $V^*$ is non-convex, as $V^c$
my be hard to compute explicitly. In such cases, the
Response-Based Approachability algorithm
developed in this paper offers a convenient alternative, as it only requires to
compute at each stage a response of the agent to a certain mixed action of the opponent, relative to the original goal function $V^*$. As seen below,
this computation typically requires the solution of an optimization problem, which
is inherent in the definition of $V^*$.

We next specialize the discussion to certain concrete models of interest.

\begin{algorithm}[h!]
\caption{Generalized No-Regret Algorithm}
\label{algo:gen_no_regret}
\begin{trivlist}
\item \textbf{Input:} Desired reward sets, represented by the multifunction $V^*:\Delta(\Z) \rightarrow \RR^K$.
\item \textbf{Initialization:} At time step $n=1$, use arbitrary mixed action $p_1$, and set arbitrary values $v^*_1 \in \RR^{K}$, $q^*_1 \in \Delta(\Z)$.

\item \textbf{At time step $n = 2, 3, ...$}:
\begin{enumerate}
\item Set
\[
\lambda_{v, n-1} = \bar{v}^*_{n-1} - \bar{v}_{n - 1}, \quad \lambda_{q, n-1} = \bar{q}^*_{n-1} - \bar{q}_{n - 1},
\]
where
\[
\bar{v}^*_{n-1} = \frac{1}{n-1}\sum_{k = 1}^{n-1} v^*_k, \quad \bar{q}^*_{n-1} = \frac{1}{n-1}\sum_{k = 1}^{n-1} q^*_k.
\]
\item Solve the following zero-sum matrix game:
    \[
    p_n \in \argmax_{p \in \Delta(\A)} \min_{q \in \Delta(\Z)} \l(\lambda_{v, n-1} \cdot v(p, q) + \lambda_{q, n-1} \cdot q \r),
    \]
    \[
    q^*_n \in \argmin_{q \in \Delta(\Z)} \max_{p \in \Delta(\A)} \l(\lambda_{v, n-1} \cdot v(p, q) + \lambda_{q, n-1} \cdot q \r).
    \]
\item Choose action $a_n$ according to $p_n$.
\item Pick $p^*_n$ such that $v\l(p^*_n, q^*_n\r) \in V^*(q^*_n)$,
and set
\[
v^*_n = v(p^*_n, q^*_n).
\]
\end{enumerate}
\end{trivlist}
\end{algorithm}

\subsection{Global Cost Functions} \label{sec:global}

The following problem of regret minimization with global cost functions was
introduced in \citet{lit:global}.
Suppose that the goal of the agent is to minimize a
 general (i.e., non-linear) function of the average reward vector $\bar{V}_n$. In particular, we are given a continuous function $G:\RR^{K} \rightarrow \RR$, and the goal is to minimize $G(\bar{V}_n)$.
For example, $G$ may be some \emph{norm} of $\bar{V}_n$. We  define the best-cost-in-hindsight, given a mixed action $q$ of the opponent, as
\[
G^*(q) \triangleq \min_{p \in \Delta(\A)} G(v(p, q)),
\]
so that the satisficing payoff set may be defined as
\[
V^*(q) = \l\{v \in {\cal V}_0: G(v) \leq G^*(q) \r\},
\]
where ${\cal V}_0=\conv\{v(a,z)\}_{a\in\A, z\in\Z}$ is the set of feasible reward
vectors.
Clearly, the agent's response to $q$ is any mixed action that minimizes
$G(v(p,q))$, namely
\begin{equation} \label{eqn:resp_global}
p^*(q) \in \argmin_{p \in \Delta(\A)} G(v(p, q)).
\end{equation}

By Proposition \ref{prop:GenRegret}, the convex hull $V^c$ of $V^*$ is attainable
by the agent. The relation between $V^c$ and $V^*$ can be seen depend on convexity properties of $G$ and $G^*$. In particular, we have the following immediate result (a slight extension of \cite{lit:global}).
\begin{proposition}\label{prop:Vconvexity}
\begin{itemize}\item[]
\item[(i)] For $q\in\Delta(\Z)$,
\[
V^c(q) \subset \l\{v \in {\cal V}_0 \,:\, \conv(G)(v) \leq \conc(G^*)(q) \r\},
\]
where $\conv(G)$ and $\conc(G^*)$ are the lower convex hull of $G$ and the upper concave hull of $G^*$, respectively.
\item[(ii)] Consequently, if $G(v)$ is a convex function over $v\in{\cal V}_0$,
then the relaxed goal function $\conc(G^*)(q)$ is attainable.
\item[(iii)] If, furthermore, $G^*(q)$ is a concave function of $q$, then
$V^c=V^*$, and the goal function $G^*(q)$ is attainable.
\end{itemize}
\end{proposition}

Clearly, if $G^*(q)$ is not concave, the attainable goal function is weaker than the
original one. Still, this relaxed goal is meaningful, at least in cases
where $G(v)$ is convex (case $(ii)$ above), so that $\conc(G^*)(q)$ is attainable. Noting that
$G^*(q) \leq \max_{q'}\min_{p} G(v(p,q'))$, it follows
that
\begin{equation}\label{eqn:security}
\conc(G^*)(q) \leq \max_{q'\in\Delta(\Z)}\min_{p\in\Delta(\A)} G(v(p,q'))
\leq \min_{p\in\Delta(\A)} \max_{q'\in\Delta(\Z)} G(v(p,q')) \,.
\end{equation}
The latter min-max bound is just the security level of the agent in the repeated
game, namely the minimal value of $G(\bar{V}_n)$ that can be secured
(as $n\to\infty)$ by playing a {\em fixed} (non-adaptive) mixed action $q'$.
Note that the second inequality in Equation \eqref{eqn:security} will be
strict except for special cases where the min-max theorem holds for $G(v(p,q))$
(which is hardly expected if $G^*(q)$ is non-concave).

Convexity of $G(v)$ depends directly on its definition, and will hold for
cases of interest such as norm functions. Concavity of $G^*(q)$,
on the other hand, is more demanding and will hold only in special cases.
We give below two examples, one in which $G(v)$ is convex but $G^*(q)$
is not necessarily concave, and one in which both properties are satisfied.
In the next subsection we discuss an example where neither is true.

\begin{example}[Absolute Value] {\rm
Let $v:\A \times \Z \rightarrow \RR$ be a scalar reward function,
and suppose that we wish to minimize the deviation of the average reward
$\bar{V}_n$ from a certain set value, say 0.
Define then $G(v) = |v|$, and note that $G$ is a convex function. Now,
\[
G^*(q) \triangleq \min_{p \in \Delta(\A)} |v(p, q)| =
\left\{ \begin{array}{lcl}
          \min_{a\in\A} v(a,q) & : & v(a,q)> 0,\, a\in \A  \\
          \min_{a\in\A}(- v(a,q)) & : & v(a,q) < 0,\, a\in \A \\
          0 & : & \text{otherwise}
        \end{array}\right.
\]
The response $p^*(q)$ of the agent is obvious from these relations.
We can observe two special case in this example:
\begin{itemize}
\item[$(i)$] The problem reduces to the classical no-regret problem if
the rewards $v(a,z)$ all have the same sign (positive or negative), as the absolute
value can be removed. Indeed, in this case $G^*(q)$ is concave, as a minimum
of linear functions.
\item[$(ii)$]
 If the set $\{v(a,q),a\in\A\}$ includes elements of opposite signs (0 included)
for each $q$, then $G^*=0$, and the point $v=0$ becomes attainable.
\end{itemize}
In general, however, $|v(p, q)|$ may be a {\em strictly} convex function of $q$
for a fixed $p$, and the minimization above need not lead to a concave function.
In that case, we can ensure only the attainability of $\conc(G^*)(q)$.

We note that the computation of $\conc(G^*)$ may be fairly complicated in general,
which implies the same for computing the projection onto the associated goal set
$S=\{(v,q): |v| \leq \conc(G^*)(q)\}$. However, these computations are not needed
 in the proposed Response-Based Approachability algorithm, where the required computation of the agent's response $p^*(q)$ is straightforward.
} \end{example}

\begin{example}[Load Balancing] {\rm
The following model was considered in \citet{lit:global}, motivated by load
balancing and job scheduling problems.
Consider a scalar loss function $\ell:\A \times \Z \rightarrow \RR$, with $\ell(a,z)\geq 0$, and define a corresponding vector-valued reward function $v:\A \times \Z \rightarrow \RR^{|\A|}$, where $v(a, z)$ has $\ell(a, z)$ at entry $a$ and $0$ otherwise:
\[
v(a, z)[a'] = \begin{cases}
\ell(a, z), & a = a'\\
0, & \text{otherwise.}
\end{cases}
\]
The average reward vector then represents the the average loss of the agent on different actions. Namely,
\[
\bar{V}_n[a] = \frac{1}{n} \sum_{k=1}^n \I\l\{a = a_k \r\} \ell(a, z_k).
\]
Also, note that
\[
v(p, q) = \{p(a)\ell(a, q) \}_{a \in \A} \triangleq p \odot \ell(\cdot, q),
\]
and
\[
G^*(q) = \min_{p \in \Delta(\A)} G(v(p, q)) = \min_{p \in \Delta(\A)} G\l(p \odot \ell(\cdot, q)\r).
\]
\citet{lit:global} analyzed the case where $G$ is either the $d$-norm with $d > 1$, or the infinity norm (the {\em makespan}). Clearly $G$ is convex here.
Furthermore, it was shown that the function
\[
F^*(\ell) \triangleq \min_{p \in \Delta(\A)} G\l(p \odot \ell\r)
\]
is concave in $\ell$.
Now, since $G^*(q) = F^*(\ell(\cdot, q))$ and $\ell(\cdot, q)$ is linear in $q$, then $G^*(q)$ is also concave in $q$.

The agent's response is easily computed for this problem:
The response $p=p^*(q)$ is generally mixed, with $p_a$ proportional
to $\ell(a,q)^{-d/(d-1)}$ for $d<\infty$, and to $\ell(a,q)^{-1}$ for the
infinity norm.
} \end{example}

\subsection{Reward-to-Cost Maximization} \label{sec:ratio}

Consider the repeated game model as before, where the goal of the agent is to maximize the ratio $\bar{U}_n/\bar{C}_n$. Here, $\bar{U}_n$ is the average of a scalar reward function $u(a, z)$ and $\bar{C}_n$ is the average of a scalar positive cost function $c(a, z)$.
This problem is mathematically equivalent to the problem of regret minimization in repeated games with variable stage duration considered in \citet{lit:nahum_var} (in that paper, the cost was specifically taken as the stage duration).
Observe that this problem is a particular case of the global cost function problem presented in Section \ref{sec:global}, with vector-valued payoff function $v(a, z) = (u(a, z), c(a, z))$ and $G(v) = -u/c$. However, here $G(v)$ is not convex in $v$. We will therefore need to apply
specific analysis in order to obtain similar bounds to those of Proposition \ref{prop:Vconvexity}$(ii)$.

We mention that similar bounds to the ones established below were obtained in \citet{lit:nahum_var}. The algorithm there was based on playing a best-response to calibrated forecasts of the opponent's mixed actions. As mentioned in the introduction, obtaining these forecasts is computationally hard, and the present formulation offers a considerably less demanding alternative.

Denote
$$
\rho(a,q) \triangleq \frac{u(a,q)}{c(a,q)}, \quad \rho(p,q) \triangleq \frac{u(p,q)}{c(p,q)}.
$$
and let
$$
\mathrm{val}(\rho) \triangleq
\max_{p \in \Delta(\A)} \min_{q \in \Delta(\Z)}  \rho(p,q)
= \min_{q \in \Delta(\Z)} \max_{p \in \Delta(\A)} \rho(p,q)
$$
(the last equality is proved in the above-mentioned paper; note that $\rho(p,q)$
is not generally concave-convex). It may be seen that $\mathrm{val}(\rho)$ is
the value of the zero-sum repeated game with payoffs $\bar{U}_n/\bar{C}_n$,
hence serves as a security level for the agent. A natural goal for the agent would
be to improve on $\mathrm{val}(\rho)$ whenever the opponent's actions deviate (in terms
of their empirical mean) from the minimax optimal strategy.

Let
\[
\rho^*(q) \triangleq \max_{p \in \Delta(A)} \rho(p,q)
\]
denote the \emph{best ratio-in-hindsight}.
We apply Algorithm \ref{algo:gen_no_regret}, with $v=(u,c)$ and the satisficing payoff set
\[
V^*(q) = \l\{v = (u, c): \text{ } \frac{u}{c} \geq \rho^*(q) \r\}
\]
(observe that both $\rho^*(q)$ and $V^*(q)$ are non-convex functions in general).
The agent's response is given by any mixed action
$$
p^*(q) \in P^*(q) \triangleq \argmax_{p \in \Delta(\A)} \rho(p, q).
$$
It is easily verified that the maximum can always be obtained here in pure actions (\cite{lit:nahum_var}; see also the proof of Prop.~\ref{prop:ratio} below).
Denote
$$
A^*(q) \triangleq \argmax_{a\in \A} \rho(a,q),
$$
and define the following relaxation of $\rho^*(q)$:
\begin{equation} \label{eqn:rho_relaxed}
{\rho_1}(q) \triangleq \inf\l\{\frac{\sum_{j=1}^J u(a_j,q_j)}{\sum_{j=1}^J c(a_j,q_j)} \,:\,
1\leq J <\infty, \, q_j \in \Delta(\Z), \, \frac{1}{J} \sum_{j=1}^{J} q_j =q,\, a_j\in A^*(q_j) \r\}.
\end{equation}
Clearly, ${\rho_1}(q) \leq \rho^*(q)$. We will show below that ${\rho_1}$ is attained by applying Algorithm \ref{algo:gen_no_regret} to this problem. First, however, we compare $\rho_1$ to the security level $\mathrm{val}(\rho)$.
%
\vspace*{0.1cm}
\begin{lemma} \label{lem:ratio_relaxed_properties}
\begin{itemize}
\item[]
\item[(i)] $ {\rho_1}(q) \geq \mathrm{val}(\rho)$ for all $q\in\Delta(\Z)$.
\item[(ii)] $ {\rho_1}(q) > \mathrm{val}(\rho)$ whenever
${\rho^*}(q) > \mathrm{val}(\rho)$.
\item [(iii)]  If $q$ corresponds to a pure action $z$, namely $q=\delta_z$,
then $ {\rho_1}(q) =  \rho^*(q)$.
\item[(iv)]$ {\rho_1}(q)$ is a continuous function of $q$.
\end{itemize}
\end{lemma}
\begin{proof}
To prove this Lemma, we first derive a more convenient expression for ${\rho_1}(q)$.
For $a \in \A$, let
\[
Q_a \triangleq \l\{q \in \Delta(\Z)\, : \, a \in A^*(q) \r\}
\]
denote the (closed) set of mixed actions to which $a$ is a best-response action. Observe that for given $J$, $q_1, ..., q_J$ and $a_j \in A^*(q_j)$, we have
\[
\frac{\sum_{j=1}^J u(a_j,q_j)}{\sum_{j=1}^J c(a_j,q_j)} = \frac{\sum_{a \in \A} N_a u(a, \bar{q}_a)}{\sum_{a \in \A} N_a c(a, \bar{q}_a )},
\]
where
\[
N_a = \sum_{j = 1}^J \I \l\{a_j = a \r\},
\quad
\bar{q}_a = \frac{1}{N_a} \sum_{j = 1}^J \I \l\{a_j = a \r\} q_j\,.
\]
Note that $\bar{q}_a \in \conv(Q_a)$ as it is a convex combination of $q_j \in Q_a$. Therefore, the definition in \eqref{eqn:rho_relaxed} is equivalent to
\[
{\rho_1}(q) = \min\l\{\frac{\sum_{a \in \A} \alpha_a u(a, q_a)}{\sum_{a \in \A} \alpha_a c(a, q_a)} \, : \, \alpha \in \Delta(\A), q_a \in \conv(Q_a), \sum_{a \in \A} \alpha_a q_a = q \r\}.
\]
Now, this is exactly the definition of the so-called \emph{calibration envelope} in \citet{lit:nahum_var}, and the claims of the lemma follow by Lemma 6.1 and Proposition 6.4 there.
\end{proof}

It may be seen that ${\rho_1}(q)$ does not fall below the security level
$\mathrm{val}(q)$, and is strictly above it when $q$ is not a minimax action
with respect to $\rho(p,q)$.
Furthermore, at the vertices vertices of $\Delta(\Z)$, it actually coincides with the best ratio-in-hindsight $\rho^*(q)$.

We proceed to the following result that proves the attainability of ${\rho_1}(q)$.

\begin{proposition} \label{prop:ratio}
Consider Algorithm \ref{algo:gen_no_regret} applied to the model of the present Subsection. Suppose that the response action to $q^*_n$ is chosen as any action $p^*_n \in P^*(q^*_n)$ and consequently the target point is set to $v^*_n = \l(u(p^*_n, q^*_n), c(p^*_n, q^*_n)\r)$. Then,
$$
\liminf_{n\to\infty} \l(\frac{\bar{U}_n}{\bar{C}_n} - {\rho_1}(\bar{q}_n)\r) \geq 0 \quad (a.s.)
$$
for any strategy of the opponent.
\end{proposition}

\begin{proof}
Algorithm \ref{algo:gen_no_regret} ensures that, with probability 1,
\begin{equation}\label{eqn:q}
\l\|\bar{q}_n - \bar{q}^*_n \r\| \to 0,
\end{equation}
\begin{equation}\label{eqn:uc}
\l|\bar{U}_n - \frac{1}{n}\sum_{k = 1}^n u(p^*_k,q^*_k) \r| \to 0,\quad
\l|\bar{C}_n - \frac{1}{n}\sum_{k = 1}^n c(p^*_k,q^*_k) \r|\to 0;
\end{equation}
see Theorem \ref{theo:algo} and recall the asymptotic equivalence of smoothed and
non-smoothed averages.
Noting that the cost $c$ is positive and bounded away zero, \eqref{eqn:uc} implies that
\begin{equation} \label{eqn:ratio_limit}
\lim_{n \to \infty} \frac{\bar{U}_n}{\bar{C}_n}  = \lim_{n \to \infty}  \frac{\sum_{k = 1}^n r(p^*_k,q^*_k)}{\sum_{k = 1}^n c(p^*_k,q^*_k)}.
\end{equation}
Let
\begin{equation} \label{eqn:rho_relaxed_smooth}
{\rho_2}(q) \triangleq \inf\l\{\frac{\sum_{j=1}^J u(p_j,q_j)}{\sum_{j=1}^J c(p_j,q_j)} \,:\,
1\leq J <\infty, \, q_j \in \Delta(\Z), \, \frac{1}{J} \sum_{j=1}^{J} q_j =q,\, p_j\in P^*(q_j) \r\}.
\end{equation}
Clearly,
\begin{equation} \label{eqn:rho_lower}
\frac{\sum_{k = 1}^n r(p^*_k,q^*_k)}{\sum_{k = 1}^n c(p^*_k,q^*_k)} \geq {\rho_2}(\bar{q}^*_n).
\end{equation}

Also, it may be verified that the infimum in (\ref{eqn:rho_relaxed_smooth}) is obtained in pure actions $a_j \in A^*(q_j)$, implying that
\begin{equation} \label{eqn:pure_act}
{\rho_2}(q) = {\rho_1}(q).
\end{equation}
Indeed, note that
\[
\frac{\sum_{j=1}^J u(p_j,q_j)}{\sum_{j=1}^J c(p_j,q_j)} \leq K
\]
is equivalent to
\[
\sum_{j=1}^J u(p_j,q_j) - K\sum_{j=1}^J c(p_j,q_j) \leq 0.
\]
Now, consider minimizing the left-hand-side over $p_j \in P^*(q_j)$. Due to the linearity in $p_j$ and the fact that $P^*(q_j)$ is just the mixture of actions in $A^*(q_j)$, the optimal actions are pure (that is, in $A^*(q_j)$).

Combining \eqref{eqn:ratio_limit}, \eqref{eqn:rho_lower}, and \eqref{eqn:pure_act}, we obtain
$$
\liminf_{n\to\infty} \l(\frac{\bar{U}_n}{\bar{C}_n} - {\rho_1}(\bar{q}^*_n)\r) \geq 0.
$$
The proof is concluded by using \eqref{eqn:q} and the continuity of ${\rho_1}$ (see Lemma \ref{lem:ratio_relaxed_properties}).
\end{proof}

\subsection{Constrained Regret Minimization}

We finally address the constrained regret minimization problem,
introduced in \citet{lit:shie_constr}. Here, in addition to
the scalar reward function $u$, we are given a \emph{vector-valued} cost function
$c:\A \times \Z \rightarrow \RR^{s }$. We are also given a closed and convex set $\Gamma \subseteq \RR^{s }$, the constraint set,
which specifies the allowed values for the long-term average cost.
A specific common case is that of a linear constraint on each cost
component, that is
$
\Gamma = \l\{c \in \RR^{s }: \text{ } c_i \leq \gamma_i, \text{ } i=1,...,s  \r\}
$
for some given vector $\gamma \in \RR^{s }$.
The constraint set is assumed to be {\em feasible} (or \emph{not excludable}), in the sense that
for every $q \in \Delta(\Z)$, there exists $p \in \Delta(\A)$ such that
$c(p, q) \in \Gamma$.

Let
$
\bar{C}_n \triangleq n^{-1} \sum_{k = 1}^n c_k
$
denote, as before, the average cost by time $n$. The agent is required
to satisfy the cost constraints, in the sense that
$\lim_{n \rightarrow \infty} d(\bar{C}_n, \Gamma) = 0$ must hold, irrespectively of the opponent's
play.
Subject to these constraints, the agent wishes to maximize its average reward $\bar{U}_n$.

We observe that a concrete learning application for the constrained regret minimization problem was proposed in \citet{lit:nips_me}. There, we
considered the on-line problem of merging the output of multiple binary classifiers,
with the goal of maximizing the true-positive rate, while
keeping the false-positive rate under a given threshold $0 < \gamma < 1$.
As shown in that paper, this problem may be formulated as a constrained regret minimization problem.

A natural extension of the best-reward-in-hindsight $u^*(q)$ in (\ref{eqn:rih_basic}) to the constrained setting is given by
\begin{equation}
\label{def:cbe}
u^*_{\Gamma}(q) \triangleq \max_{p \in \Delta(\A)} \l\{u(p, q) \,:\, c(p, q) \in \Gamma \r\}.
\end{equation}
We can now define the satisficing payoff set of the pairs $v = (u, c) \in \RR^{1 + s}$ in terms of $u^*_{\Gamma}(q)$ and $\Gamma$:
\[
V^*(q) \triangleq \l\{v = (u, c) \in \RR^{1 + s}: u \geq u^*_{\Gamma}(q), c \in \Gamma\r\}.
\]
Note that $u^*_{\Gamma}(q)$ is \emph{not} convex in general, and consequently $V^*(q)$ is not convex as well.
Indeed, it was shown in \citet{lit:shie_constr} that $V^*(q)$ is not approachable in general.
The convex hull of $V^*(q)$ may be written as
\begin{equation} \label{eqn:crm_ch}
V^c(q) = \l\{(u, c) \in \RR^{s + 1}: \text{ } u \geq \conv \l(u^*_{\Gamma}\r)(q), \text{ } c \in \Gamma \r\},
\end{equation}
where the function $\conv \l(u^*_{\Gamma}\r)$ is the lower convex hull of $u^*_{\Gamma}$.

Two algorithms were proposed in \citet{lit:shie_constr} for attaining $V^c(q)$. The first is a standard (Blackwell's) approachability algorithm for $S = \{(v,q)\,:\, v \in V^c(q)\}$, which requires the demanding calculation of projection directions to $S$.
The second algorithm employs a best-response to calibrated forecasts of the opponent's mixed actions. As mentioned in the introduction, obtaining these forecasts is computationally hard. In contrast, our algorithm only requires the computation of the response $p^*(q)$ as any maximizing action in (\ref{def:cbe}).
Similarly to the case of global cost functions, step 4 of Algorithm \ref{algo:gen_no_regret} boils down to solving the optimization problem in
\eqref{def:cbe} for $q = q^*_n$. Note that $p^*_n$ can be efficiently computed for a given $q^*_n$ since \eqref{def:cbe} is a convex program in general, while it is a linear program whenever the constraints set is a polyhedron.

\begin{remark}
Note that since $V^c(q)$ is unbounded in the direction of its first coordinate $u$, the algorithm variant presented in Subsection \ref{sec:inf} can be applied. In this case, the first coordinate of the steering direction ${\lambda}_n$
can be set to zero in $\tilde{\lambda}_n$  whenever it is negative,
which corresponds to $\bar{u}_{n-1} \geq \bar{u}^*_{n-1}$, thereby avoiding
an unnecessary reduction in $\bar{u}_{n-1}$. Similarly, for linear constraint sets
of the form $\{c_i\leq \gamma_i\}$, the $c_i$-coordinate of ${\lambda}_n$ may be nullified whenever $[\bar{c}_{n-1}]_i \leq [\bar{c}^*_{n-1}]_i$.

A similar modification can be applied also in the reward-to-cost problem of Section \ref{sec:ratio}. That is, the $u$-coordinate of ${\lambda}_n$
can be set to zero whenever $\bar{u}_{n-1} \geq \bar{u}^*_{n-1}$, while the $c$-coordinate of ${\lambda}_n$ may be nullified whenever $\bar{c}_{n-1} \leq \bar{c}^*_{n-1}$.
\end{remark}

\section{Conclusion} \label{sec:conc}

We have introduced in this paper a class of approachability algorithms that are based on Blackwell's dual, rather than primal, approachability condition.
The proposed algorithms rely directly on the availability of a response function,
rather than projection onto the goal set (or related geometric quantities),
and are therefore convenient
in certain problems where the latter may be hard to compute.
At the same time, the additional computational requirements are generally comparable
to those of the standard Blackwell algorithm and its variants.

The proposed algorithms were applied to a class of generalized no-regret problems,
that includes reward-to-cost maximization, and reward maximization subject to
average-cost constraints.
The resulting algorithms are apparently the first computationally efficient algorithms in this generalized setting.

In this paper we have focused on a repeated \emph{matrix game} model, where the
action sets of the agent and the adversary in the stage game are both discrete.
It is worth pointing out that the essential results of this paper also apply
directly to models with \emph{convex} action sets are convex,
say $x\in X$ and $y\in Y$,
and the vector reward function $r(x,y)$ is bilinear in its arguments.
In that case the (observed) actions $x$ and $y$ simply take the place of
the mixed actions $p$ and $q$, leading to similar algorithms and
convergence results. The continuous-action model is of course relevant
to linear classification and regression problems.

Other extensions of possible interest for the response-based algorithms
suggested in this paper include stochastic game models, problems of
partial monitoring, and possibly nonlinear (concave-convex) reward functions.
These are left for future work.

\section*{Acknowledgements}
We wish to thank Shie Mannor for useful discussions, and for
pointing out the application to regret minimization with global
cost functions. We further thank Elad Hazan for helpful comments
on the Appendix. This research was supported by the Israel Science
Foundation grant No.\ 1319/11.

\bibliography{bibliography}




\section*{Appendix A}

We outline in this Appendix a somewhat more direct version of the
no-regret based approachability algorithms  proposed in \citet{ABH10}.
This version avoids the lifting procedure used in that paper, that
treats general (convex) target sets by lifting them to convex cones in
higher dimension.
This brief outline is meant to highlight the geometric nature and requirements of this
class of algorithms.

Let $S \subseteq \RR^{\ell}$ be the convex and closed target set to be approached. Let $h_S$ denote the \emph{support function} of $S$:
\[
h_S(\theta) \triangleq \sup_{r \in S} (\theta \cdot r), \quad \theta \in \RR^{\ell}.
\]
Note that $h_S$ is a convex function. The Euclidean distance from a point $r$ to $S$
may be expressed as
\begin{equation} \label{eqn:dist_support}
d(r, S) = \max_{\theta \in B_2(1)} \l\{\theta \cdot r - h_S(\theta) \r\},
\end{equation}
where $B_2(1)$ is the Euclidean unit ball, $B_2(1) = \{\theta \in \RR^{\ell} \, : \, \theta \cdot \theta \leq 1 \}$ (see \citet{lit:Rockafeller70}, Section 16;
this equality can also be verified directly using the minimax theroem).

Blackwell's (primal) separation condition can now be written as follows\footnote{
We use here the notations and formulation of the present paper, where $p$ and $q$
are mixed actions in their respective simplices.
However, the following observations are valid also
for the case where $p$ and $q$ are (observed) actions in bounded convex sets and
$r(p,q)$  a bilinear function thereof, as considered in \cite{ABH10}.}:
\begin{itemize}
\item
For each $\theta\in B_2(1)$ there exist $p\in\Delta(\Z)$ so that, for every
$q \in \Delta(\Z)$,
\[
 \theta \cdot r(p, q) \leq \sup_{r \in S} \theta \cdot r \equiv h_S(\theta) \,,
\]
that is,
\begin{equation} \label{eqn:black_nr}
 \theta \cdot r(p, q) - h_S(\theta) \leq 0.
\end{equation}
\end{itemize}

An approachability algorithm can be devised as follows. Observe that the function $f_{r}(\theta) = \theta \cdot r - h_S(\theta)$ is \emph{concave} in $\theta$ (for each $r$). Hence, an online concave programming algorithm applied to the sequence
of functions $(\theta\cdot r_n -h_S(\theta))$, with $r_n$ arbitrary (bounded) vectors, will produce a sequence of steering directions $\{\theta_n\}$ in $B_2(1)$  so that
\begin{equation} \label{eqn:appr_nr_bound}
\frac{1}{n} \sum_{k = 1}^n \l(\theta_k \cdot r_k - h_S(\theta_k) \r)
\geq \max_{\theta \in B_2(1)} \l\{\theta \cdot \bar{r}_n - h_S(\theta) \r\} - o(1).
\end{equation}
Now, observing (\ref{eqn:black_nr}), one can choose each $p_n$ so that
$r_n=r(p_n,z_n)$ satisfies $\theta_n \cdot r_n - h_S(\theta_n) \leq 0$.
Substituting in (\ref{eqn:appr_nr_bound}) we obtain
\[
\max_{\theta \in B_2(1)} \l\{\theta \cdot \bar{r}_n - h_S(\theta) \r\} \leq o(1).
\]
Hence, by \eqref{eqn:dist_support}, $d(\bar{r}_n, S) \leq o(1)$.

Observe that the above scheme applies an online concave programming algorithm to the functions $f_r(\theta)$, that are defined through the support function $h_S(\theta)$.
Thus, it essentially requires computing the support function $h_S$ (or its derivative) at some point in each stage of the game.

To be specific, let us apply the gradient ascent algorithm of \cite{lit:grad_covex}
to the problem. The resulting approachability algorithm proceeds as follows.
\begin{itemize}
\item[1.] At stage $n$, we start with $\theta_{n-1}$, $p_{n-1}$, $r_{n-1}$, $z_{n-1}$  from the previous stage.
\item[2.] Let
\begin{align}
\theta_n & =\text{Proj}(\theta_{n-1} -
 \eta_{n} \nabla_{\theta}(\theta_{n-1}\cdot r_{n-1} - h_S(\theta_{n-1})) \nonumber \\
 & =
 \text{Proj}(\theta_{n-1} -
 \eta_{n} (r_{n-1} - \nabla h_S(\theta_{n-1}))
 \end{align}
where $\text{Proj}$ is then projection onto the unit ball.
\item[3.]
Choose $p_n$ according to \eqref{eqn:black_nr}, so that
$\theta_n\cdot r(p_n,z) - h_S(\theta_n) \leq 0$ for all $z\in \Z$.
\item[4.]
Observe the opponent's action $z_n$, and set $r_n=r(p_n,z_n)$.
\end{itemize}

\end{document}